\documentclass{article}

\usepackage{microtype}
\usepackage{graphicx}
\usepackage{subcaption}
\usepackage{booktabs} 

\usepackage{hyperref}

\usepackage[preprint]{icml2026}

\usepackage{amsmath}
\usepackage{amssymb}
\usepackage{mathtools}
\usepackage{amsthm}

\usepackage[capitalize,noabbrev]{cleveref}

\theoremstyle{plain}
\newtheorem{theorem}{Theorem}[section]

\newtheorem{lemma}[theorem]{Lemma}

\theoremstyle{definition}
\newtheorem{definition}[theorem]{Definition}

\theoremstyle{remark}
\newtheorem{remark}[theorem]{Remark}

\usepackage[textsize=tiny]{todonotes}

\icmltitlerunning{Contract And Conquer: How to Provably Compute Adversarial Examples for a Black-Box Model?}

\begin{document}

\twocolumn[
  \icmltitle{Contract And Conquer: How to Provably Compute Adversarial Examples for a Black-Box Model?}



  \icmlsetsymbol{equal}{*}

  \begin{icmlauthorlist}
    \icmlauthor{Anna Chistyakova}{equal,yyy}
    \icmlauthor{Mikhail Pautov}{equal,xxx,yyy}
  \end{icmlauthorlist}

  \icmlaffiliation{xxx}{AXXX}
  \icmlaffiliation{yyy}{Trusted AI Research Center, RAS}

\icmlcorrespondingauthor{Anna Chistyakova}{ann244111@gmail.com}

  \icmlkeywords{adversarial robustness, black-box attacks}

  \vskip 0.3in
]



\printAffiliationsAndNotice{\icmlEqualContribution}

\begin{abstract}
  Black-box adversarial attacks are widely used as tools to test the robustness of deep neural networks against malicious perturbations of input data aimed at a specific change in the output of the model. Such methods, although they remain empirically effective, usually do not guarantee that an adversarial example can be found for a particular model. In this paper, we propose Contract And Conquer (CAC), an approach to provably compute adversarial examples for neural networks in a black-box manner. The method is based on knowledge distillation of a black-box model on an expanding distillation dataset and precise contraction of the adversarial example search space. CAC is supported by the transferability guarantee: we prove that the method yields an adversarial example for the black-box model within a fixed number of algorithm iterations. Experimentally, we demonstrate that the proposed approach outperforms existing state-of-the-art black-box attack methods on ImageNet dataset for different target models, including vision transformers.  
\end{abstract}


\begin{figure}[ht]
  \begin{center}
    \centerline{\includegraphics[width=\columnwidth]{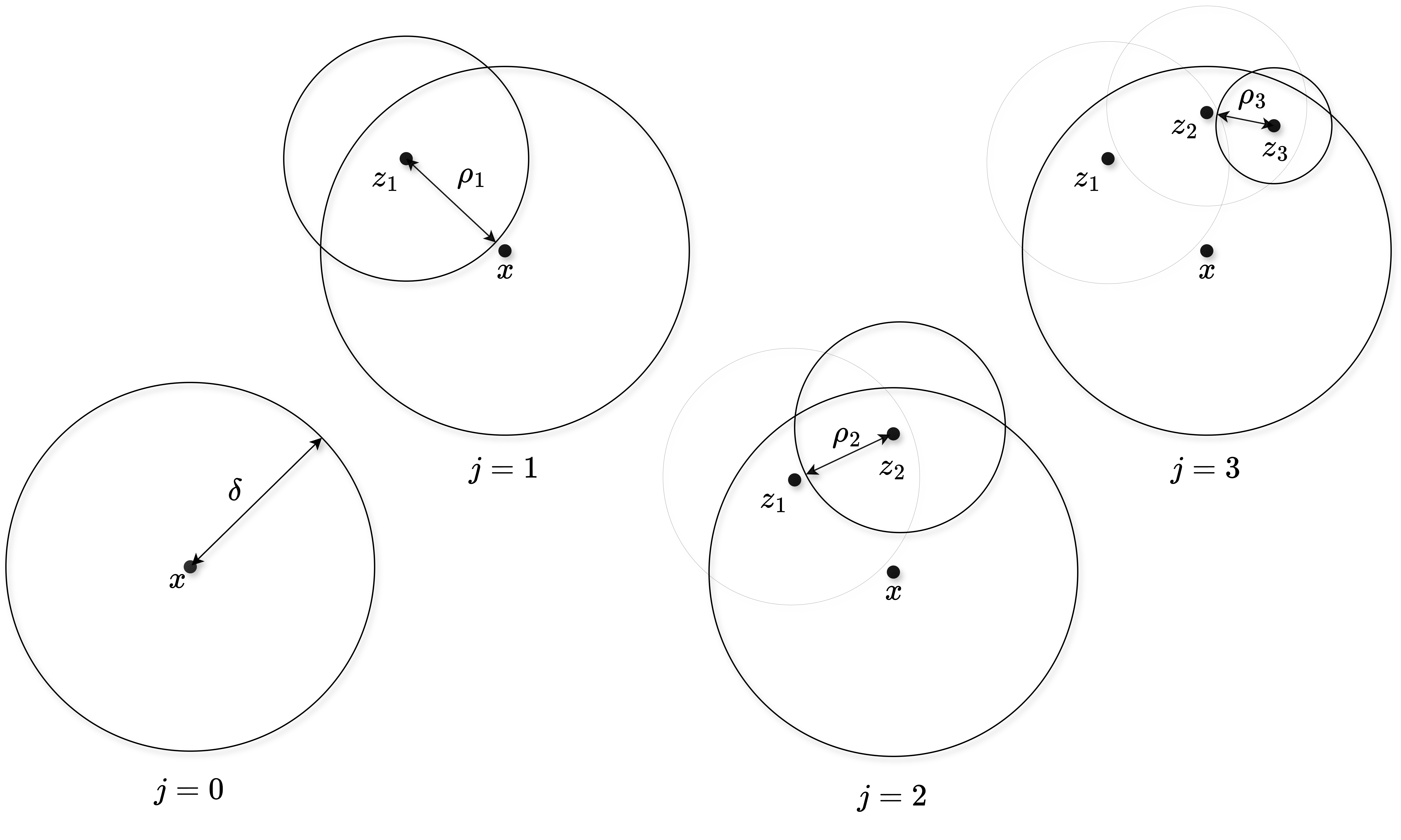}}
    \caption{
      Illustration of the contraction of the adversarial example search space. Given the number $j$ of algorithm iteration, the adversarial example search space on iteration $j$, namely, $U_\delta(x)_j,$ is the intersection of the $\rho_j-$vicinity of an adversarial example $z_j$ with the initial attack search space, $U_{\delta}(x).$ Formally, $U_\delta(x)_j = U_{\delta}(x) \cap U_{\rho_j}(z_j).$ The quantity $\rho_j$ is defined in Eq. \ref{eq:rho}. For each algorithm iteration, the adversarial example search space is represented by the intersection of bold circles. 
    }
    \label{img:contraction}
  \end{center}
\end{figure}

\section{Introduction}
Evaluating and enhancing the robustness of neural networks to malicious perturbations of input data, called adversarial attacks, is crucial in safety-critical applications, such as medicine or autonomous systems. It has long been known that a small, often imperceptible perturbation of image \cite{goodfellow2014explaining} or a minor paraphrase of an input prompt \cite{zhu2023promptrobust} can cause a desired change in the output of the corresponding model. It is noteworthy that the   effectiveness of adversarial attacks is experimentally confirmed in the black-box settings, when the attacker has limited access to the model, namely, when they can query the model and receive its output in a fixed format \cite{qi2023transaudio,maheshwary2021generating,guo2019simple}.

\begin{figure*}[ht]
  \begin{center}
    \centerline{\includegraphics[width=0.77\textwidth]{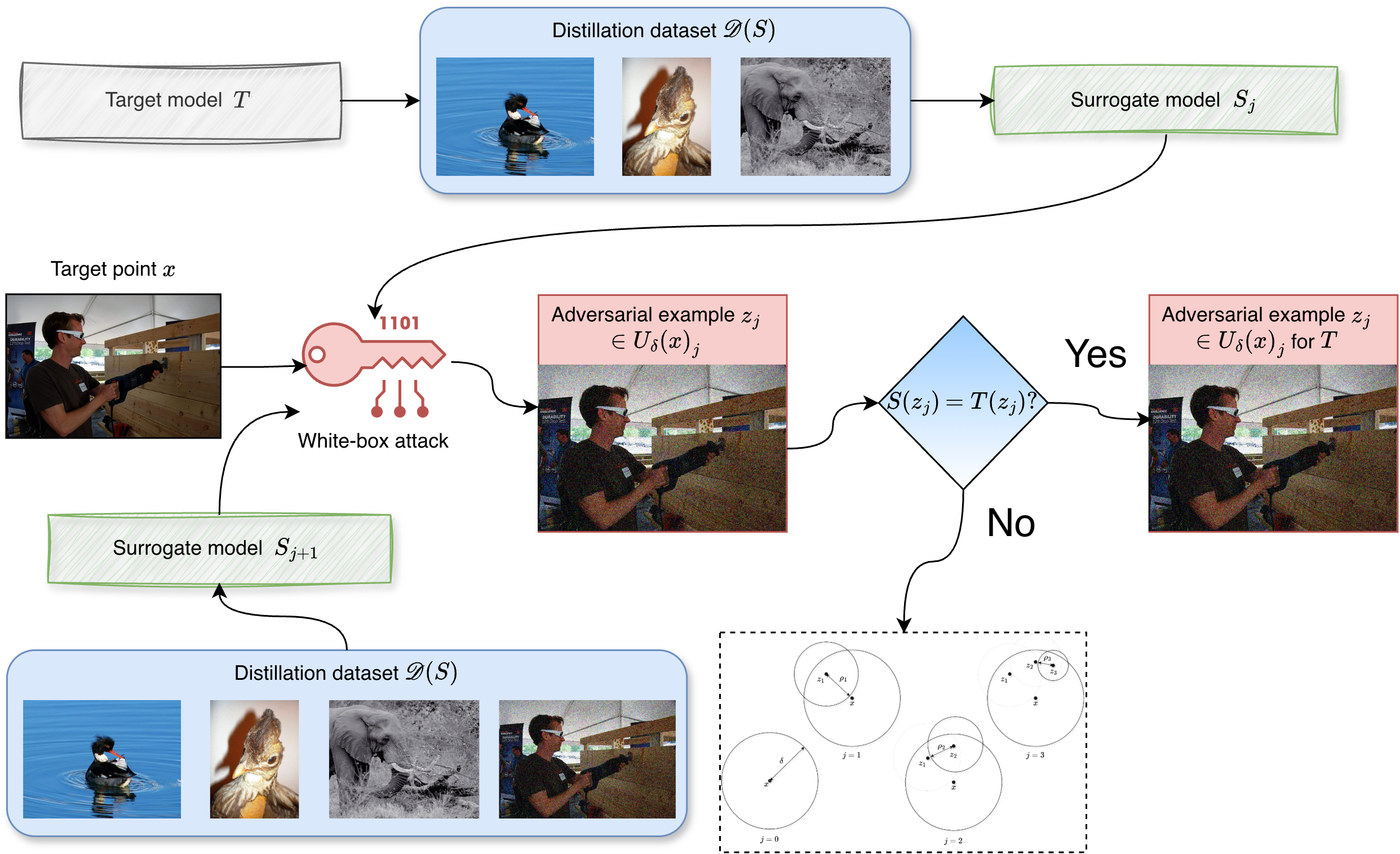}}
    \caption{
      Schematic representation of the proposed method. Given alternation iteration $j$ and the target model $T$, we prepare the distillation dataset $\mathcal{D}(S)$ and train the surrogate model $S_j$. Then, $S_j$ is attacked at the target point $x$ in the white-box setting, and an adversarial example $z_j$ is computed. If $z_j$ is transferable to $T,$ algorithm returns $z_j$ and stops; otherwise, the adversarial example search space is contracted as shown in Fig. \ref{img:contraction}, $(z_j, T(z_j))$ is added to the distillation dataset, and the next instance of the surrogate model, $S_{j+1},$ is obtained.}
    \label{img:method}
  \end{center}
\end{figure*}

Starting from the seminal work \cite{Szegedy2014intriguing}, the majority of research in the field of adversarial machine learning has focused on developing  methods to compute adversarial examples and empirical approaches to defend the models against them. Mainly, the methods of computing adversarial examples are based on utilizing the information about the target model's outputs and gradients  \cite{carlini2017towards,madry2018towards,andriushchenko2020square,park2024hard} or its estimation \cite{guo2019simple,chen2017zoo,cheng2024efficient,han2024bfs2adv}. In parallel, empirical defense methods are mainly based on adversarial training \cite{madry2018towards,bai2021adversarial}, where the model is trained on generated adversarial examples, gradient regularization \cite{ross2018improving}, calibration \cite{stutz2020confidence}, or weight perturbation \cite{wu2020adversarial,xu2022weight}. It is worth mentioning that the existence of an arms race between empirical defenses and adversarial attacks is concerning for security-critical applications: specifically, it can not be guaranteed that recently developed empirical defense mechanisms will remain effective against novel attack methods, and vice versa. Thus, the effectiveness of the application of empirical methods from adversarial machine learning to evaluate the robustness in safety-critical settings is questionable.

More than that, a variety of regulatory acts for artificial intelligence systems are in the process of development today, for example, the EU AI Act or the US National AI Initiative Act. These frameworks, among other things, are designed to develop standards of robustness of machine learning algorithms and services to adversarial attacks. As a consequence, to deploy a machine learning system in a specific setting, one will have to verify that it complies with the aforementioned standards. 

To ground the evaluation of the resilience of machine learning methods to adversarial attacks, it may be  reasonable to focus on certified robustness methods. Instead of relying on heuristics used in empirical defense approaches, certified robustness methods aim to provide mathematical guarantees about a model's behavior  when its input is subjected to a certain perturbation. The methods of certified robustness are usually based on randomized smoothing \cite{cohen2019certified,pautov2022smoothed,voracek2024treatment}, set propagation techniques \cite{gowal2018effectiveness,mao2024understanding},  convex relaxation \cite{anderson2025towards,kim2024convex}, or probabilistic certification \cite{weng2019proven,pautov2022cc,feng2025prosac}. These approaches can yield sample-level or population-level guarantees that no adversarial example exists, given the type of perturbation and the perturbation budget. Unfortunately, certified robustness comes at a cost of computationally expensive inference \cite{cohen2019certified}, may require significant changes to both training and inference, limit available model architectures \cite{cullen2025position}, or may lead to a notable performance degradation of the certifiably robust model.  Aforementioned drawbacks are among the ones that limit the embedding of certified robustness into the state-of-the-art machine learning services: for example, an integration of randomized smoothing defense into medical diagnostics or into digital services that mark harmful content may lead to a significant degradation of performance on benign input data or severely slow down the system. 
As a consequence, to both retain practical effectiveness and to align with the upcoming AI regulatory acts, the developers will probably seek alternatives to  certified robustness.   

At the same time, a complementary research question arises: how to \emph{guarantee} that the given black-box machine learning is not robust? Specifically,  a method to prove that the given model is not robust might be an important tool for the assessment of robustness, especially from the perspective of compliance with the prospective standards. In this paper, we focus on this research question and propose \emph{Contract And Conquer} (CAC), an iterative method to compute adversarial examples for black-box models with convergence guarantees. By design, CAC is an alternation of two processes: (i) knowledge distillation \cite{hinton2015distilling} of the  target black-box target model by a small  surrogate  model and (ii) a white-box adversarial attack on a surrogate model within a vicinity of the target input point. Intuition behind CAC is simple: knowledge distillation forces the surrogate model to replicate the predictions of the target  model in the closed vicinity of target point, where a white-box attack on the surrogate  model is used to craft adversarial examples; careful alternation of these operations, together with small contraction of the vicinity of the target point, yields an upper bound on the number of alternations needed to compute an adversarial example for the black-box target model. 

Our contributions are summarized as follows:
\begin{itemize}
    \item A novel iterative transfer-based adversarial attack,  Contract and Conquer (CAC), is proposed. The method is based on knowledge distillation of the target model on an expanding  dataset and the white-box attack on the surrogate model within a contracting adversarial example search space. 
    \item We theoretically demonstrate that, under mild assumptions on the surrogate model, the proposed transfer-based attack is guaranteed to yield an adversarial example for the black-box target  model within a fixed number of algorithm iterations.
    \item We experimentally show that CAC outperforms the state-of-the-art black-box attack methods on popular image benchmarks for different target models, including vision transformers. 
\end{itemize}

\section{Related Work}

\subsection{Adversarial Attacks}
Soon after the vulnerability of neural networks to adversarial perturbations was established \cite{goodfellow2014explaining,Szegedy2014intriguing}, a lot of attack methods have been proposed \cite{moosavi2016deepfool,carlini2017towards,chen2020hopskipjumpattack}. One way to categorize attack methods is based on the degree of accessibility of the target model to an adversary. White-box attacks, that imply full access to the target model, including its internal weights, gradients and/or training data, are broadly gradient-based ones \cite{goodfellow2014explaining,carlini2017towards,madry2018towards}, or surrogate loss-based ones \cite{zhang2022enhancing,zhang2022improving,wang2023diversifying}. Gradient-based attacks exploit information about the target model's gradients, and, hence, tend to be of superior effectiveness; at the same time, the transferability, or the ability of adversarial examples to generalize across models, of  gradient-based adversarial attacks is usually modest, since the examples are specifically computed  for the specific model instance \cite{madry2018towards,qin2022boosting}.  To enhance the transferability of adversarial examples, some methods redesign objective functions, utilizing, among other things, an information from the hidden layers of the target model, for example, by improving the similarity between the features of the adversarial example and its benign preimage \cite{huang2019enhancing}, enhancing an invariance of an adversarial noise w.r.t. input  objects \cite{liu2025boosting}  or by disrupting a subset of important object-aware features of the target model \cite{wang2021feature}. In contrast, black-box attacks assume that an adversary only has query access to the target model, and, hence, can be used to evaluate the robustness of  machine learning services in real-world setups \cite{papernot2017practical,zhang2021reverse,ma2025jailbreaking}. These methods can be coarsely divided into score-based \cite{uesato2018adversarial,andriushchenko2020square,bai2020improving}, decision-based \cite{rahmati2020geoda,maho2021surfree,wang2022triangle} and transfer-based \cite{liu2017delving,xie2019improving,nasee2022rimproving,li2023making,chen2024rethinking} categories. When decision-based and score-based methods utilize the outputs of the target model to conduct an attack, transfer-based ones rely on training the surrogate models to further conduct a white-box adversarial attack against. These approaches, in particular, tend to show better transferability of adversarial examples from one model to another, mainly by design of optimization procedure and due to different heuristics used  \cite{debicha2023tad,xie2025chain}. 

We want to highlight that although transfer-based black-box adversarial attacks demonstrate remarkable transferability of adversarial examples from the surrogate models to the target models, they do not provide any guarantees of the success of an attack; in general, this important disadvantage is shared by known black-box attack methods.  


\subsection{Adversarial Defenses}
To level out the threat of adversarial attacks, plenty of defense methods have been proposed. They can be divided into two categories, namely, empirical ones and certified ones. When empirical methods mainly rely on data-driven and architecture-level heuristics, the certified ones are equipped with formal guarantees: for example, they allow to formally prove that no adversarial example exists in a particular vicinity of the target point \cite{gowal2018effectiveness,cohen2019certified}. Among  empirical methods, adversarial training \cite{goodfellow2014explaining,madry2018towards} and its modifications \cite{shafahi2019adversarial,wong2020fast} stand out. These approaches enhance the robustness of neural networks by jointly training them on benign samples and adversarial examples generated by certain attack methods, exposing the population of adversarial examples that the model has to be defended from; it is noteworthy that adversarial training methods offer the strongest empirical robustness. The other methods pre-process the data before feeding it to the target network \cite{guo2018countering,nesti2021detecting}, adopt image purification techniques \cite{nie2022diffusion,wei2025real}, use auxiliary methods to detect and correct adversarial perturbations \cite{liu2019detection,aldahdooh2022adversarial,che2025large}, or modify the defended model \cite{yu2021defending,allouaha2025daptive,zhao2025universal}. Among the certified methods, randomized smoothing \cite{cohen2019certified,lecuyer2019certified} and its variants  \cite{yang2020randomized,bansal2022certified,korzh2025certification} are used to provide the state-of-the-art worst-case guarantees on robustness of neural networks in different setups. Instead of providing the output for a single input sample, these methods aggregate the predictions over a large amount of perturbed input samples. The other certified defense methods include, but are not limited to, set propagation techniques \cite{gowal2018effectiveness,wang2021beta,mao2024understanding}, and formal verification methods \cite{tjeng2019evaluating,shi2020robustness}. 

It is worth mentioning that application of provable, effective, but computationally expensive defense methods in real-world AI systems is rather selective and incremental, than rigorous  and complete, since in many setups,  speed, performance and utility may outweigh robustness.

\section{Methodology}

In this section, we provide background and motivation followed by the description of the proposed method. Later, we introduce theoretical justifications of the method.  

\subsection{Background and Motivation}
In this work, we separately consider hard-label and soft-label settings. Specifically, let $T$ be the target black-box model that takes real-valued image tensor $x \in [0,1]^d$ as input, and returns, in hard-label setting, class index $y \in [1,\dots,K]$, where $K$ is the number of classes; in soft-label setting, it returns the vector of class probabilities $p \in [1,\dots,K].$ Here and below, we represent the prediction label assigned by the black-box model $T$ for input $x$ in the form 
\begin{align}
    & T(x) = y, \text{ for hard-label case,} \nonumber \\
    & T(x) = \arg\max\limits_{i \in [1,\dots,K]}T(x)_i,  \text{ for soft-label}.
\end{align}
Let $S:[0,1]^d \to [0,1]^K$ be the white-box model that maps an input tensor to a class index as $y = \arg\max_{i \in [1,\dots,K]}S(x)_i.$
In this work, we focus on the simplest formalism of an adversarial attack given in the definitions below. 
\begin{definition}
    Let $x$ be the sample correctly classified by the model $T$, $y = T(x)$, and let $\delta>0$ be the fixed constant.  Then, the object $x': \|x-x'\|_\infty \le \delta$ is called an adversarial example for $T$ at point $x,$ if $T(x') \ne T(x).$ If $T(x')=y'$ for some predefined class $y'\ne y,$ then $x'$ is called targeted adversarial example. 
\end{definition}
Starting from here, we refer to $U_\delta(x) = \{x':\|x-x'\|_\infty \le\delta \}$ as the initial adversarial example search space. Following the well-established notion \cite{madry2018towards}, we treat the $l_\infty$ constraint as the measure of invisibility of adversarial examples.  
\begin{definition}
    Let $x'\in U_\delta(x)$ be the adversarial example computed for the white-box model $S$ at point $x$, and $T$ be the separate black-box model. Then, $x'$ is called transferable from $S$ to $T$ if 
    \begin{equation}
        \begin{cases}
        \arg\max_{i \in [1,\dots,K]}S(x)_i = T(x), \\ 
         \arg\max_{i \in [1,\dots,K]}S(x')_i = T(x').
        \end{cases}
    \end{equation}
\end{definition}
The goal of this work is to propose an approach to compute adversarial examples for the target model, $T$, that is supported by a mathematical guarantee of the success of an attack. To do so, we apply a transfer-based attack paradigm. In a nutshell, instead of computing an attack for the target model explicitly, we apply knowledge distillation to obtain a smaller surrogate model, $S$, to attack in the white-box setting; then we demonstrate experimentally and formally prove that, under mild assumptions on the surrogate model and controllable contraction of the adversarial example search space, we are \emph{guaranteed} to compute an adversarial example for $T$ within the fixed number of iterations. In the next section, we provide a detailed description of the proposed method.

\subsection{Description of CAC}
\subsubsection{Surrogate Model and White-box Attack}
\label{sec:single_alternation}
Suppose that the black-box model $T$, the target point $x$ of class $y$, and the initial adversarial attack search space  $U_\delta(x)$ are fixed. We firstly obtain the surrogate model, $S$, by applying knowledge distillation to $T.$ The  distillation dataset for the surrogate model, $\mathcal{D}(S),$ consists of pairs $(x_k, T(x_k)),$ where $\{x_k\}_{k=1}^{m-1}$ is a subset of a hold-out dataset. This subset is formed in the following way: firstly, a random subset  $\{x_k\}_{k=1}^{N_{init}}$  is sampled from a hold-out dataset; then, among $N_{init}$ points, we choose $m-1$ closest ones to the target point $x$. The target point $(x, T(x))$ is included in $\mathcal{D}(S).$ Consequently, knowledge distillation is performed by training $S$ on $\mathcal{D}(S)$ by minimizing an empirical risk 
\begin{equation}
    L(S, \mathcal{D}(S)) = \frac{1}{|\mathcal{D}(S)|} \sum_{(x_k, y_k) \in \mathcal{D}(S)} l(S, x_k, y_k),
\end{equation}
where $l(S,x_k,y_k)$ is the cross-entropy loss function. In the experiments, we use $N_{init}=10000$ and $m=300$. 

We assume that the surrogate model has enough learning capabilities to match the predictions of the target model on $\mathcal{D}(S),$ which is formalized in the following form:
\begin{equation}
\label{eq:student_gap}
    \begin{cases}
        T(x_k) = \arg\max_{i \in [1,\dots,K]}S(x_k)_i  = y_k,\\
        \frac{1}{2}\left[S(x_k)_{y_k} - \max_{i\ne y_k} S(x_k)_i\right] > \varepsilon
    \end{cases}
\end{equation}
for all $(x_k, y_k) \in \mathcal{D}(S).$ Here, the second inequality reflects the confidence of the surrogate model, and $\varepsilon>0$ is a constant.

When the surrogate model is trained, we attack it in a white-box manner. Specifically, we apply MI-FGSM \cite{dong2018boosting} to find an adversarial example for $S$ within initial adversarial attack search space $U_\delta(x) $: 
\begin{equation}
\label{eq:wb}
    \begin{cases}
        g_{t+1}  = \mu g_t + \frac{\nabla_{x'_t}l'(S, x'_t, y)}{\|\nabla_{x'_t}l'(S, x'_t, y)\|_1},\\
        
        x'_{t+1} = \text{Proj}_{U_{\delta}(x)}\left[x'_t + \alpha(\delta) \text{ sign} (g_{t+1})\right], \\
        x'_1 = x,\qquad g_1=0. \\
    \end{cases}
\end{equation}
Here, $\mu$ is the momentum parameter, $\alpha(\delta)>0$ is the gradient step, $\text{Proj}_{U_{\delta}(x)}$ is the  projection onto the attack search space, $M$ is the maximum number of gradient steps, and  $l'$ is a loss function specified later. We refer to the process of distillation followed by the search for an adversarial example for the surrogate model as a single alternation. 

\subsubsection{Adjustment of Attack Parameters}
\label{sec:adjustment}
Let $j$ be the number of current alternation. We assume that for some iteration number $t \in [1,\dots,M]$,  an adversarial example for the surrogate model, $z_j=x_t',$ is found.  Then, the target model $T$ is queried with $z_j$ to check if it is transferable from $S$ to $T.$ If so, an algorithm yields $z_j$ as an adversarial example for $T;$ otherwise, we adjust the adversarial attack procedure: firstly,  $(z_j, T(z_j))$ is included in the distillation dataset $\mathcal{D}(S)$; secondly, the adversarial example search space is contracted as follows: 
\begin{align}
\label{eq:contraction}
    &U_\delta(x)_j \gets U_{\delta}(x) \cap U_{\rho_j}(z_j),
\end{align}
where 
\begin{equation}
\label{eq:rho}
     \rho_j = t \|z_j-z_{j-1}\|_\infty
\end{equation}
is the contracted distance between two previous adversarial examples. 
Here, $U_\delta(x)_j$ is the adversarial example search space after $j-$th alternation iteration,  $t \in (0,1)$ is the contraction parameter, $U_{\rho_j}(z_j) = \{a: \|a-z_j\|_\infty \le \rho_j\}$ and $z_0 =x.$ After these two adjustments, an algorithm proceeds to the next alternation described in Section \ref{sec:single_alternation}, but with updated distillation dataset and adversarial example search space. The procedures from Sections \ref{sec:single_alternation} and \ref{sec:adjustment} are described in Algorithm \ref{alg:adv_whitebox}. The adversarial example search space contraction is schematically presented in Figure \ref{img:contraction}.

\begin{algorithm}[htb]
   \caption{Contract and Conquer}
   \label{alg:adv_whitebox}
\begin{algorithmic}[1]
    \REQUIRE Black-box target model $T$, target point $x$ of class $y$, distance threshold $\delta$,  momentum parameter $\mu$, maximum number of MI-FGSM iterations $M$, maximum number queries to the target model $N$, initial size of distillation dataset $m$, hold-out dataset data points $\{x_k\}_{k=1}^{m-1}$, contraction parameter $t$
    \ENSURE Surrogate model $S$, adversarial example $(z, T(z))$ for the target model $T$
    \STATE {$\mathcal{D}(S) \gets \{(x_k, T(x_k))\}_{i=1}^{m-1} \cup \{(x,y)\}$}
    \COMMENT{initialize distillation dataset}
    \STATE{$N \gets N - m$}
    \COMMENT{the remaining number of queries to the target model decreases since $m$ were spent to initialize distillation dataset}
    %
    \STATE{$z_0 \gets x, \quad U_\delta(x)_0 \gets U_\delta(x), \quad \alpha \gets \delta/M$}
    \STATE{$j \gets 1$}
    \WHILE{$N \ge 0$}
        \STATE {$\text{train S on distillation dataset } \mathcal{D}(S)$}
        \STATE {$(z_j, \arg\max_{i \in [1,\dots,K]}S(z_j)_i) \gets \text{MI-FGSM}(S,\alpha, \mu, U_\delta(x)_{j-1},M, (x,y))$ \label{alg_line:wb}}
        \COMMENT{compute an adversarial example for the surrogate model according to Eq. \ref{eq:wb}}
        \IF{$\arg\max_{i \in [1,\dots,K]}S(z_j)_i = h(T, z_j)$}
        \STATE{\textbf{return} $S, \ (z_j, h(T,z_j))$}
        \ELSE 
        \STATE{$\mathcal{D}(S) \gets \mathcal{D}(S) \cup \{(z_j, T(z_j))\}$}
        \STATE{$\rho_j \gets t\|z_j - z_{j-1}\|$}
        \STATE{$U_\delta(x)_j \gets U_{\delta}(x) \cap U_{\rho_j}(z_j) $}
        \COMMENT{contract adversarial example search space according to Eq. \ref{eq:contraction}}
        \STATE{$\alpha \gets \rho_j / M$}
        \COMMENT{update the gradient step}
        \ENDIF
        \STATE{$N \gets N - 1$}
        \COMMENT{the remaining number of queries decreases since $1$ query is spent to compute $T(z_j)$}
        \STATE{$j \gets j + 1$}
\ENDWHILE
\end{algorithmic}
\end{algorithm}
\begin{remark}
    CAC is, in fact, not tied to a specific white-box attack; the usage of MI-FGSM is motivated by its simplicity and efficiency. The procedure in Algorithm \ref{alg:adv_whitebox} is described for a single white-box adversarial example to ease the notation. In practice, the algorithm computes a batch of $n_{adv}=10$ adversarial examples for speed-up. To ensure the variety of these examples, each example is computed for the target point $z_0 = x+\varepsilon_k$ and search space $U_\delta(x).$ Here $\{\varepsilon_k\}_{k=1}^{n_{adv}} \sim \mathcal{U}[-\delta, \delta].$ Additionally, if an adversarial example for the target model is found and the maximum number of queries to the target model has not been exhausted, (i) the radius of 
    the initial adversarial example search space, $\delta$, decreases, and (ii) the algorithm restarts to possibly yield an adversarial example closer to the target point.

\end{remark}

\subsection{Convergence Guarantee}

In this Section, we introduce the theoretical justification of CAC and justify the assumptions made. The following lemma represents the convergence guarantee of the method.
\begin{lemma}
\label{th:lemma}
    Fix an input sample $x$ and initial adversarial attack search space, $U_\delta(x) = \{a: \|x-a\|_\infty \le \delta\}.$   
    Suppose that for every $j \in \mathbb{Z}_{+},$ the white-box attack in Algorithm \ref{alg:adv_whitebox} yields an adversarial example for the model $S$.  Let  $S$   be a differentiable function with the bounded gradients in $U_\delta(x)$ for every $j \in \mathbb{Z}_+$ and let 
        \begin{equation}
        \gamma = \sup_{j \in \mathbb{Z}_+}\ \sup_{x' \in U_\delta(x)} \|\nabla S(x')\|_{op,\infty},
    \end{equation}
    where $\|\cdot\|_{op, \infty}$ is the operator norm induced by $l_\infty$ norm of vectors. 
    Let the surrogate model $S$ be trained according to Eq. \ref{eq:student_gap}, meaning that if $y_k = \arg\max_{i \in [1,\dots,K]}S(x_k)_i,$ then 
    \begin{equation}
    \label{eq:half_difference}
       \frac{1}{2} \left[ S(x_k)_{y_k} - \max_{i \ne y_k} (S(x_k))_i\right] > \varepsilon
    \end{equation} for all $(x_k,y_k) \in \mathcal{D}(S)$. Then, Algorithm \ref{alg:adv_whitebox} yields an adversarial example for the model $S$ which is transferable to $T$ at most at $(n-1)-$th alternation iteration, where 
    \begin{equation}
        (n-1) \ln t \le \ln \varepsilon - \ln\delta -\ln\gamma. 
    \end{equation}
\end{lemma}

\begin{remark}
{The proof is moved to the appendix, not to distract the reader.}  
Here, we want to briefly motivate  assumptions made in Lemma \ref{th:lemma}. Firstly, the boundedness of the gradient $S$ is achieved by construction of $S$ out of layers with the bounded gradients and by using activation functions with the bounded gradients, what is done in our case. Secondly, the assumptions about the learning capabilities of the surrogate model formalized in Eq. \ref{eq:student_gap} and the possibility to compute an adversarial example for the surrogate model on each alternation iteration can be achieved simultaneously by an appropriate choice of the architecture of $S$ and its training; these two assumptions are practically verifiable. 
\end{remark}


\section{Experiments}
In this section, we provide technical description of experiments, datasets and model architectures, baseline methods, and the comparison methodology.

\subsection{Setup of Experiments}

\subsubsection{Datasets and Target Models}

In our experiments, we use CIFAR-10  \cite{krizhevsky2009learning} and ImageNet \cite{deng2009imagenet} datasets to train the surrogate models. For the baseline experiments, we choose ResNet-50 \cite{he2016deep} and ViT-B \cite{dosovitskiy2021an} architectures of target models.  
The accuracy of ResNet50 on ImageNet is  $80.13\%$, on CIFAR-10 is $94.65\%$; the accuracy of ViT-B on ImageNet is $85.21\%$, on CIFAR-10 is $96.89\%$.

\subsubsection{Surrogate Models and White-Box Attack}

We use ResNet-18 as the architecture of the white-box surrogate model. The knowledge distillation is conducted for $100$ epochs with the use of Adam optimizer with the constant learning rate of $10^{-3}$. 
We conduct the white-box attack on the surrogate models with the following parameters: the number of MI-FGSM iterations is set to be $M=3$,  the momentum parameter of attack is set to be $\mu=1.0,$ the contraction parameter is set to be $t=0.99,$ the initial adversarial example search space radius is set to be $\delta = 0.125$,
the gradient step is set to be $\alpha = \delta / M.$ The loss function $l'(S,x'_t,y)$ from Eq. \ref{eq:wb}  is the cross-entropy loss in the hard-label setting and MSE loss for the soft-label setting. To quantitatively  evaluate the effectiveness of the method, we randomly choose the subset of $100$ target points from the test subset of the corresponding dataset which are initially correctly classified by the target model. 


\subsubsection{Baseline Methods}

We evaluate the proposed method against HopSkipJump \cite{chen2020hopskipjumpattack}, Sign-OPT \cite{cheng2020sign}, GeoDA \cite{rahmati2020geoda},  SquareAttack \cite{andriushchenko2020square}, SparseRS \cite{croce2022sparse}, PAR \cite{shi2022decision} and  AdvViT \cite{zhou2025query} methods. HopSkipJump, Sign-OPT, and GeoDA are regarded as query-efficient competitive benchmarks in the hard-label black-box setting;  SparseRS and SquareAttack are among the most efficient in the soft-label setting. At the same time,  AdvViT and PAR are the state-of-the-art hard-label black-box attacks designed specifically for transformer architectures. Additionally, we evaluate CAC against combinations of HopSkipJump and SignOPT with PAR, where the latter is used as an initialization for the baseline methods. The hyperparameters of the baseline methods are reported in the appendix.

\begin{table*}[htb]
  \caption{Quantitative comparison of attack methods, hard-label setting, the target model is ResNet-50, the dataset is  ImageNet. }
  \label{tab:results_resnet50_hard_image}
  \begin{center}
      \begin{sc}
        \begin{tabular}{@{}lllllll@{}}
        \toprule
        Method &  ASR & AQN & Avg $l_2$ & Std $l_2$ & Avg $l_\infty$ & Std $l_\infty$   \\ \midrule
        CAC (ours) & $1.00$ & $487.95$ & $\underline{35.074}$ & $18.833$ & $\underline{0.153}$ & $0.080$    \\ 
        HopSkipJump $l_2$ & $1.00$ & $500.31$ & $48.838$ &  $29.118$ & $0.539$ & $0.280$    \\ 
        HopSkipJump $l_\infty$ & $1.00$ & $500.01$ & $73.255$ & $35.856$ & $0.361$ & $0.202$   \\ 
        SignOPT  & $1.00$ & $548.24$ & $48.047$ & $28.467$ & $0.551$ & $0.283$   \\ 
        GeoDA  & $1.00$ & $524.98$  &  $49.658$ & $31.117$ & $0.180$ & $0.094$    \\ 
        \bottomrule
        \end{tabular}
      \end{sc}
  \end{center}
\end{table*}

\begin{table*}[!ht]
  \caption{Quantitative comparison of attack methods, hard-label setting, the target model is ViT-B, the dataset is ImageNet. }
  \label{tab:results_vitb_hard_image}
  \begin{center}
      \begin{sc}
        \begin{tabular}{@{}lllllll@{}}
        \toprule
        Method & ASR & AQN & Avg $l_2$ & Std $l_2$ & Avg  $l_\infty$ & Std $l_\infty$ \\ \midrule
        CAC  (ours) & $1.00$ & $488.91$ & $49.282$ & $26.488$ & $\underline{0.165}$ & $0.091$    \\   
        HopSkipJump $l_2$ & $1.00$ & $500.34$ & $70.122$ & $38.343$ & $0.685$ & $0.318$  \\
        HopSkipJump $l_\infty$ &  $1.00$ & $500.01$ & $106.142$ & $48.455$ & $0.563$ & $0.292$   \\
        SignOPT  & $1.00$ & $557.31$ & $74.744$ & $44.850$ & $0.708$ & $0.338$ \\ 
        GeoDA  & $1.00$ & $540.21$ & $65.471$ & $40.497$ & $0.190$ & $0.124$  \\
        PAR  & $1.00$ & $322.38$   & $38.751$ &  $25.745$ & $0.889$ & $0.233$  \\ 
        AdvViT & $0.75$ & $461.04$ & $\underline{34.520}$ & $20.257$ & $0.584$ & $0.301$   \\ 
        SignOPT + PAR & $1.00$ & $467.64$ & $51.468$ & $37.941$ & $0.625$ & $0.276$  \\ 
        HopSkipJump $l_2$ + PAR & $1.00$ & $500.36$ & $56.514$ & $40.454$ & $0.665$ & $0.328$   \\ 
        HopSkipJump $l_\infty$ + PAR & $1.00$ & $500.09$  & $102.909$ & $49.018$ & $0.543$ & $0.287$  \\ 
        \bottomrule
        \end{tabular}
      \end{sc}
  \end{center}
\end{table*}

\begin{table*}[!htb]
  \caption{Quantitative comparison of attack methods, soft-label setting, the target model is ResNet-50, the dataset is ImageNet. }
  \label{tab:results_resnet50_soft_image}
  \begin{center}
      \begin{sc}
        \begin{tabular}{@{}lllllll@{}}
        \toprule
        Method & ASR & AQN & Avg $l_2$ & Std $l_2$ & Avg $l_\infty$ & Std $l_\infty$  \\ \midrule
        CAC (ours) & $1.00$ & $489.93$ & $\underline{36.396}$ & $19.038$ & $\underline{0.122}$ & $0.068$   \\ 
        SquareAttack $l_\infty$ &  $0.98$ & $500.00$ & $89.292$ & $4.953$ & $0.250$ & $0.000$   \\ 
        SparseRS  & $0.94$ & $500.00$  & $44.470$ & $2.574$ & $0.994$ & $0.017$  \\ 
        \bottomrule
        \end{tabular}
      \end{sc}
  \end{center}
\end{table*}

\begin{table*}[!htb]
  \caption{Quantitative comparison of attack methods, soft-label setting, the target model is ViT-B, the dataset is ImageNet. }
  \label{tab:results_vitb_soft_image}
  \begin{center}
      \begin{sc}
        \begin{tabular}{@{}lllllll@{}}
        \toprule
        Method &  ASR & AQN & Avg $l_2$ & Std $l_2$ & Avg $l_\infty$ & Std $l_\infty$    \\ \midrule
        CAC (ours) & $1.00$ & $488.60$ & $\underline{41.370}$ &  $23.579$ & $\underline{0.144}$ & $0.084$    \\
        SquareAttack $l_\infty$ & $0.26$ & $500.00$ & $90.103$ & $4.602$ & $0.250$ & $0.000$    \\ 
        SparseRS  & $0.79$ & $500.00$ & $44.335$ & $2.444$ & $0.993$ & $0.017$   \\
        \bottomrule
        \end{tabular}
      \end{sc}
  \end{center}
\end{table*}


\begin{table*}[htb]
  \caption{Quantitative comparison of attack methods, hard-label setting, the target model is ResNet-50, the dataset is CIFAR-10. }
  \label{tab:results_resnet50_hard_cifar}
  \begin{center}
      \begin{sc}
        \begin{tabular}{@{}lllllll@{}}
        \toprule
        Method &  ASR  & AQN & Avg $l_2$ & Std $l_2$ & Avg $l_\infty$ & Std $l_\infty$ \\ \midrule
        CAC (ours) & $1.00$ & $291.0$ & $\underline{2.675}$ & $1.091$ & $\underline{0.061}$ & $0.025$  \\ 
        HopSkipJump $l_2$ & $1.00$ & $300.07$ & $2.704$ & $2.634$ & $0.174$ & $0.161$  \\ 
        HopSkipJump $l_\infty$ & $1.00$ & $310.66$ & $3.281$ & $3.232$ & $0.082$ & $0.085$   \\ 
        SignOPT  & $0.92$ & $288.59$  & $3.642$ & $3.351$ & $0.242$ & $0.209$  \\ 
        GeoDA  & $0.96$ & $300.81$  & $3.388$ & $3.440$ & $0.071$ & $0.071$  \\ 
        \bottomrule
        \end{tabular}
      \end{sc}
  \end{center}
\end{table*}

\begin{table*}[!ht]
  \caption{Quantitative comparison of attack methods, hard-label setting, the target model is ViT-B, the dataset is CIFAR-10. }
  \label{tab:results_vitb_hard_cifar}
  \begin{center}
      \begin{sc}
        \begin{tabular}{@{}llllllll@{}}
        \toprule
        Method & ASR & AQN & Avg $l_2$ & Std $l_2$ & Avg  $l_\infty$ & Std $l_\infty$  \\ \midrule
        CAC (ours) & $1.00$ & $489.89$ & $21.625$ & $11.990$ & $\underline{0.070}$ & $0.044$    \\ 
        HopSkipJump $l_2$ & $0.99$ & $496.30$ & $40.160$ & $33.460$ & $0.417$ & $0.281$ \\
        HopSkipJump $l_\infty$ &  $1.00$ & $500.01$ & $60.742$ & $40.929$ & $0.292$ & $0.226$ \\
        SignOPT  & $0.96$ & $532.76$ & $40.653$ & $33.664$ & $0.426$ & $0.265$  \\ 
        GeoDA  & $0.94$ & $604.32$ & $25.871$ & $23.517$ & $0.071$ & $0.078$ \\
        PAR  & $1.00$ & $281.56$ & $20.526$ & $17.515$ & $0.645$ & $0.236$    \\ 
        AdvViT & $0.96$ & $530.21$ & $\underline{17.741}$ & $16.191$ & $0.319$ & $0.215$    \\ 
        SignOPT + PAR & $1.00$ & $481.18$ & $26.968$ & $27.324$ & $0.454$ & $0.221$  \\ 
        HopSkipJump $l_2$ + PAR & $1.00$ & $500.20$ & $30.352$ & $25.589$ & $0.438$ & $0.244$  \\ 
         HopSkipJump $l_\infty$ + PAR & $1.00$ & $500.04$ & $53.656$ & $33.299$ & $0.253$ & $0.180$  \\ 
        \bottomrule
        \end{tabular}
      \end{sc}
  \end{center}
\end{table*}

\begin{table*}[!htb]
  \caption{Quantitative comparison of attack methods, soft-label setting, the target model is ResNet-50, the dataset is CIFAR-10. }
  \label{tab:results_resnet50_soft_cifar}
  \begin{center}
      \begin{sc}
        \begin{tabular}{@{}lllllll@{}}
        \toprule
        Method & ASR & AQN & Avg $l_2$ & Std $l_2$ & Avg $l_\infty$ & Std $l_\infty$  \\ \midrule
        CAC (ours)  & $1.00$ & $291.00$ & $\underline{2.468}$ & $1.075$ & $\underline{0.056}$ & $0.025$  \\ 
        SquareAttack $l_\infty$ & $0.82$ & $300.00$ & $13.028$ & $0.637$ & $0.250$ & $0.000$ \\ 
        SparseRS  & $0.96$ & $300.00$ & $4.371$ & $0.348$ & $0.920$ & $0.065$   \\ 
        \bottomrule
        \end{tabular}
      \end{sc}
  \end{center}
\end{table*}

\begin{table*}[!htb]
  \caption{Quantitative comparison of attack methods, soft-label setting, the target model is ViT-B, the dataset is CIFAR-10. }
  \label{tab:results_vitb_soft_cifar}
  \begin{center}
      \begin{sc}
        \begin{tabular}{@{}llllllll@{}}
        \toprule
        Method &  ASR & AQN & Avg $l_2$ & Std $l_2$ & Avg $l_\infty$ & Std $l_\infty$   \\ \midrule
        CAC (ours) & $1.00$ & $489.50$ & $\underline{15.745}$ & $9.850$ & $\underline{0.050}$ & $0.037$  \\ 
        SquareAttack $l_\infty$ & $0.85$ & $500.00$ & $92.182$ & $3.57$ & $0.250$ & $0.000$  \\ 
        SparseRS  & $0.98$ & $500.00$ & $43.198$ & $1.86$ & $0.974$ & $0.032$ \\ 
        \bottomrule
        \end{tabular}
      \end{sc}
  \end{center}
\end{table*}


\subsubsection{Comparison Methodology}
To align CAC with the baseline methods for comparison, we fix the maximum number of queries to the target model and the initial adversarial examples search space for each target point and evaluate the efficiency for each method by computing its attack success rate. We report average distances between the target point and the closest corresponding adversarial example, as well as the average number of queries, AQN, required to compute an adversarial example at the target point. Average number of queries denotes the number of requests to the target model used by a method to generate an adversarial example for the target point, averaged over all target points. Attack success rate is the fraction of target points for which a method successfully computes an adversarial example within the  maximum number of queries. For all the methods, except the CAC, we soften the maximum number of queries to the target model: specifically, we terminate the method after the iteration during which the maximum number of queries was  exceeded.


\subsection{Results of Experiments}


We report the results separately for soft-label and hard-label case, different architectures of target models, and datasets. In Tables \ref{tab:results_resnet50_hard_image}, \ref{tab:results_vitb_hard_image}, \ref{tab:results_resnet50_soft_image}, \ref{tab:results_vitb_soft_image} we report aforementioned quantities for the subset of ImageNet and indicate, where applicable, what type of norm constraint was used in internal procedures of the methods (specifically, $l_2$ or $l_\infty$). In Tables \ref{tab:results_resnet50_hard_cifar}, \ref{tab:results_vitb_hard_cifar}, \ref{tab:results_resnet50_soft_cifar}, \ref{tab:results_vitb_soft_cifar} we report the results for CIFAR-10. We highlight the best results in terms of closeness of adversarial examples to the target points. From Tables \ref{tab:results_resnet50_hard_image} -- \ref{tab:results_vitb_soft_cifar} it can be seen that CAC yields adversarial examples closer to the initial target points than other methods in experimental setups  in terms of $l_\infty$ norm and almost all setups in terms of $l_2$ norm. At the same time, been supported by the convergence guarantee, the method shows a high attack success rate; it should be mentioned that although the other methods show high success rates as well, they are not supported by formal guarantees.



\section{Conclusion and Future Work}
In this paper, we propose Contract and Conquer, a framework to compute adversarial perturbations for  black-box neural networks with  convergence guarantees. We conduct an attack in the transfer-based paradigm. Specifically, we apply knowledge distillation to obtain a smaller surrogate model to attack in the white-box setting. We theoretically show that, under mild  assumptions on the surrogate model and controllable contraction of the adversarial examples search space, the method is guaranteed to yield an adversarial example for the target black-box model within a fixed number of iterations. Experimentally, we demonstrate that the method both shows a high attack success rate and  yields adversarial examples from a smaller vicinity of the target points than the concurrent methods. Future work includes the reduction of the influence of practical assumptions, specifically, the possibility to compute an adversarial example for the surrogate model on each algorithm iteration, to build a theoretical framework to assess the compliance of AI models with the prospective robustness standards.

\section*{Impact Statement}


This paper presents work whose goal is to advance the field of Machine
Learning. There are many potential societal consequences of our work, none
which we feel must be specifically highlighted here.


\nocite{langley00}

\bibliography{example_paper}

@article{goodfellow2014explaining,
  title={Explaining and harnessing adversarial examples},
  author={Goodfellow, Ian J and Shlens, Jonathon and Szegedy, Christian},
  journal={arXiv preprint arXiv:1412.6572},
  year={2014}
}

@inproceedings{zhu2023promptrobust,
  title={Promptrobust: Towards evaluating the robustness of large language models on adversarial prompts},
  author={Zhu, Kaijie and Wang, Jindong and Zhou, Jiaheng and Wang, Zichen and Chen, Hao and Wang, Yidong and Yang, Linyi and Ye, Wei and Zhang, Yue and Gong, Neil and others},
  booktitle={Proceedings of the 1st ACM Workshop on Large AI Systems and Models with Privacy and Safety Analysis},
  pages={57--68},
  year={2023}
}

@inproceedings{qi2023transaudio,
  title={Transaudio: Towards the transferable adversarial audio attack via learning contextualized perturbations},
  author={Qi, Gege and Chen, Yuefeng and Zhu, Yao and Hui, Binyuan and Li, Xiaodan and Mao, Xiaofeng and Zhang, Rong and Xue, Hui},
  booktitle={ICASSP 2023-2023 IEEE International Conference on Acoustics, Speech and Signal Processing (ICASSP)},
  pages={1--5},
  year={2023},
  organization={IEEE}
}

@inproceedings{maheshwary2021generating,
  title={Generating natural language attacks in a hard label black box setting},
  author={Maheshwary, Rishabh and Maheshwary, Saket and Pudi, Vikram},
  booktitle={Proceedings of the AAAI Conference on Artificial Intelligence},
  volume={35},
  number={15},
  pages={13525--13533},
  year={2021}
}

@inproceedings{guo2019simple,
  title={Simple black-box adversarial attacks},
  author={Guo, Chuan and Gardner, Jacob and You, Yurong and Wilson, Andrew Gordon and Weinberger, Kilian},
  booktitle={International Conference on Machine Learning},
  pages={2484--2493},
  year={2019},
  organization={PMLR}
}

@inproceedings{Szegedy2014intriguing,
  author       = {Christian Szegedy and
                  Wojciech Zaremba and
                  Ilya Sutskever and
                  Joan Bruna and
                  Dumitru Erhan and
                  Ian J. Goodfellow and
                  Rob Fergus},
  title        = {Intriguing properties of neural networks},
  booktitle    = {International Conference on Learning Representations},
  year         = {2014}
}

@inproceedings{moosavi2016deepfool,
  title={Deepfool: a simple and accurate method to fool deep neural networks},
  author={Moosavi-Dezfooli, Seyed-Mohsen and Fawzi, Alhussein and Frossard, Pascal},
  booktitle={Proceedings of the IEEE Conference on Computer Vision and Pattern Recognition},
  pages={2574--2582},
  year={2016}
}

@inproceedings{carlini2017towards,
  title={Towards evaluating the robustness of neural networks},
  author={Carlini, Nicholas and Wagner, David},
  booktitle={2017 IEEE Symposium on Security and Privacy (sp)},
  pages={39--57},
  year={2017},
  organization={IEEE}
}

@inproceedings{madry2018towards,
  title={Towards Deep Learning Models Resistant to Adversarial Attacks},
  author={Madry, Aleksander and Makelov, Aleksandar and Schmidt, Ludwig and Tsipras, Dimitris and Vladu, Adrian},
  booktitle={International Conference on Learning Representations},
  year={2018}
}

@inproceedings{chen2017zoo,
  title={Zoo: Zeroth order optimization based black-box attacks to deep neural networks without training substitute models},
  author={Chen, Pin-Yu and Zhang, Huan and Sharma, Yash and Yi, Jinfeng and Hsieh, Cho-Jui},
  booktitle={Proceedings of the 10th ACM Workshop on Artificial Intelligence and Security},
  pages={15--26},
  year={2017}
}

@inproceedings{cheng2024efficient,
  title={Efficient Black-box Adversarial Attacks via Bayesian Optimization Guided by a Function Prior},
  author={Cheng, Shuyu and Miao, Yibo and Dong, Yinpeng and Yang, Xiao and Gao, Xiao-Shan and Zhu, Jun},
  booktitle={International Conference on Machine Learning},
  pages={8163--8183},
  year={2024},
  organization={PMLR}
}

@article{han2024bfs2adv,
  title={BFS2Adv: black-box adversarial attack towards hard-to-attack short texts},
  author={Han, Xu and Li, Qiang and Cao, Hongbo and Han, Lei and Wang, Bin and Bao, Xuhua and Han, Yufei and Wang, Wei},
  journal={Computers \& Security},
  volume={141},
  pages={103817},
  year={2024},
  publisher={Elsevier}
}

@inproceedings{park2024hard,
  title={Hard-label based small query black-box adversarial attack},
  author={Park, Jeonghwan and Miller, Paul and McLaughlin, Niall},
  booktitle={Proceedings of the IEEE/CVF Winter Conference on Applications of Computer Vision},
  pages={3986--3995},
  year={2024}
}

@inproceedings{andriushchenko2020square,
  title={Square attack: a query-efficient black-box adversarial attack via random search},
  author={Andriushchenko, Maksym and Croce, Francesco and Flammarion, Nicolas and Hein, Matthias},
  booktitle={European Conference on Computer Vision},
  pages={484--501},
  year={2020},
  organization={Springer}
}

@inproceedings{bai2021adversarial,
  author       = {Tao Bai and
                  Jinqi Luo and
                  Jun Zhao and
                  Bihan Wen and
                  Qian Wang},
  title        = {Recent Advances in Adversarial Training for Adversarial Robustness},
  booktitle    = {International Joint Conference on Artificial
                  Intelligence},
  pages        = {4312--4321},
  publisher    = {ijcai.org},
  year         = {2021},
  doi          = {10.24963/IJCAI.2021/591},
}

@inproceedings{ross2018improving,
  title={Improving the adversarial robustness and interpretability of deep neural networks by regularizing their input gradients},
  author={Ross, Andrew and Doshi-Velez, Finale},
  booktitle={Proceedings of the AAAI Conference on Artificial Intelligence},
  volume={32},
  number={1},
  year={2018}
}

@inproceedings{stutz2020confidence,
  title={Confidence-calibrated adversarial training: Generalizing to unseen attacks},
  author={Stutz, David and Hein, Matthias and Schiele, Bernt},
  booktitle={International Conference on Machine Learning},
  pages={9155--9166},
  year={2020},
  organization={PMLR}
}

@article{wu2020adversarial,
  title={Adversarial weight perturbation helps robust generalization},
  author={Wu, Dongxian and Xia, Shu-Tao and Wang, Yisen},
  journal={Advances in Neural Information Processing Systems},
  volume={33},
  pages={2958--2969},
  year={2020}
}

@inproceedings{xu2022weight,
  title={Weight perturbation as defense against adversarial word substitutions},
  author={Xu, Jianhan and Li, Linyang and Zhang, Jiping and Zheng, Xiaoqing and Chang, Kai-Wei and Hsieh, Cho-Jui and Huang, Xuan-Jing},
  booktitle={Findings of the Association for Computational Linguistics: EMNLP 2022},
  pages={7054--7063},
  year={2022}
}

@inproceedings{cohen2019certified,
  title={Certified adversarial robustness via randomized smoothing},
  author={Cohen, Jeremy and Rosenfeld, Elan and Kolter, Zico},
  booktitle={International Conference on Machine Learning},
  pages={1310--1320},
  year={2019},
  organization={PMLR}
}

@article{pautov2022smoothed,
  title={Smoothed embeddings for certified few-shot learning},
  author={Pautov, Mikhail and Kuznetsova, Olesya and Tursynbek, Nurislam and Petiushko, Aleksandr and Oseledets, Ivan},
  journal={Advances in Neural Information Processing Systems},
  volume={35},
  pages={24367--24379},
  year={2022}
}

@article{voracek2024treatment,
  title={Treatment of statistical estimation problems in randomized smoothing for adversarial robustness},
  author={Voracek, Vaclav},
  journal={Advances in Neural Information Processing Systems},
  volume={37},
  pages={133464--133486},
  year={2024}
}

@article{gowal2018effectiveness,
  title={On the effectiveness of interval bound propagation for training verifiably robust models},
  author={Gowal, Sven and Dvijotham, Krishnamurthy and Stanforth, Robert and Bunel, Rudy and Qin, Chongli and Uesato, Jonathan and Arandjelovic, Relja and Mann, Timothy and Kohli, Pushmeet},
  journal={arXiv preprint arXiv:1810.12715},
  year={2018}
}

@article{anderson2025towards,
  title={Towards optimal branching of linear and semidefinite relaxations for neural network robustness certification},
  author={Anderson, Brendon G and Ma, Ziye and Li, Jingqi and Sojoudi, Somayeh},
  journal={Journal of Machine Learning Research},
  volume={26},
  number={81},
  pages={1--59},
  year={2025}
}

@inproceedings{kim2024convex,
  title={Convex Relaxations of ReLU Neural Networks Approximate Global Optima in Polynomial Time},
  author={Kim, Sungyoon and Pilanci, Mert},
  booktitle={International Conference on Machine Learning},
  pages={24458--24485},
  year={2024},
  organization={PMLR}
}

@inproceedings{weng2019proven,
  title={PROVEN: Verifying robustness of neural networks with a probabilistic approach},
  author={Weng, Lily and Chen, Pin-Yu and Nguyen, Lam and Squillante, Mark and Boopathy, Akhilan and Oseledets, Ivan and Daniel, Luca},
  booktitle={International Conference on Machine Learning},
  pages={6727--6736},
  year={2019},
  organization={PMLR}
}

@inproceedings{pautov2022cc,
  title={Cc-cert: A probabilistic approach to certify general robustness of neural networks},
  author={Pautov, Mikhail and Tursynbek, Nurislam and Munkhoeva, Marina and Muravev, Nikita and Petiushko, Aleksandr and Oseledets, Ivan},
  booktitle={Proceedings of the AAAI Conference on Artificial Intelligence},
  volume={36},
  number={7},
  pages={7975--7983},
  year={2022}
}

@inproceedings{feng2025prosac,
  title={PROSAC: Provably Safe Certification for Machine Learning Models under Adversarial Attacks},
  author={Feng, Chen and Liu, Ziquan and Zhi, Zhuo and Bogunovic, Ilija and Gerner-Beuerle, Carsten and Rodrigues, Miguel},
  booktitle={Proceedings of the AAAI Conference on Artificial Intelligence},
  volume={39},
  number={3},
  pages={2933--2941},
  year={2025}
}

@inproceedings{cullen2025position,
  title={Position: Certified Robustness Does Not (Yet) Imply Model Security},
  author={Cullen, Andrew Craig and Montague, Paul and Erfani, Sarah Monazam and Rubinstein, Benjamin IP},
  booktitle={Forty-second International Conference on Machine Learning Position Paper Track},
  year={2025}
}

@inproceedings{chen2020hopskipjumpattack,
  title={Hopskipjumpattack: A query-efficient decision-based attack},
  author={Chen, Jianbo and Jordan, Michael I and Wainwright, Martin J},
  booktitle={2020 IEEE Symposium on Security and Privacy (SP)},
  pages={1277--1294},
  year={2020},
  organization={IEEE}
}

@inproceedings{wang2023diversifying,
     title={{Diversifying the High-level Features for better Adversarial Transferability}},
     author={Zhiyuan Wang and Zeliang Zhang and Siyuan Liang and Xiaosen Wang},
     booktitle={Proceedings of the British Machine Vision Conference},
     year={2023},
}

@inproceedings{zhang2022improving,
  title={Improving adversarial transferability via neuron attribution-based attacks},
  author={Zhang, Jianping and Wu, Weibin and Huang, Jen-tse and Huang, Yizhan and Wang, Wenxuan and Su, Yuxin and Lyu, Michael R},
  booktitle={Proceedings of the IEEE/CVF conference on computer vision and pattern recognition},
  pages={14993--15002},
  year={2022}
}

@inproceedings{zhang2022enhancing,
  title={Enhancing the transferability of adversarial examples with random patch.},
  author={Zhang, Yaoyuan and Tan, Yu-an and Chen, Tian and Liu, Xinrui and Zhang, Quanxin and Li, Yuanzhang},
  booktitle={IJCAI},
  volume={8},
  pages={13},
  year={2022}
}

@article{qin2022boosting,
  title={Boosting the transferability of adversarial attacks with reverse adversarial perturbation},
  author={Qin, Zeyu and Fan, Yanbo and Liu, Yi and Shen, Li and Zhang, Yong and Wang, Jue and Wu, Baoyuan},
  journal={Advances in Neural Information Processing Systems},
  volume={35},
  pages={29845--29858},
  year={2022}
}

@inproceedings{huang2019enhancing,
  title={Enhancing adversarial example transferability with an intermediate level attack},
  author={Huang, Qian and Katsman, Isay and He, Horace and Gu, Zeqi and Belongie, Serge and Lim, Ser-Nam},
  booktitle={Proceedings of the IEEE/CVF International Conference on Computer Vision},
  pages={4733--4742},
  year={2019}
}

@article{liu2025boosting,
  title={Boosting the Local Invariance for Better Adversarial Transferability},
  author={Liu, Bohan and Wang, Xiaosen},
  journal={arXiv preprint arXiv:2503.06140},
  year={2025}
}

@inproceedings{wang2021feature,
  title={Feature importance-aware transferable adversarial attacks},
  author={Wang, Zhibo and Guo, Hengchang and Zhang, Zhifei and Liu, Wenxin and Qin, Zhan and Ren, Kui},
  booktitle={Proceedings of the IEEE/CVF International Conference on Computer Vision},
  pages={7639--7648},
  year={2021}
}

@inproceedings{papernot2017practical,
  title={Practical black-box attacks against machine learning},
  author={Papernot, Nicolas and McDaniel, Patrick and Goodfellow, Ian and Jha, Somesh and Celik, Z Berkay and Swami, Ananthram},
  booktitle={Proceedings of the 2017 ACM on Asia Conference on Computer and Communications Security},
  pages={506--519},
  year={2017}
}

@inproceedings{zhang2021reverse,
  title={Reverse attack: Black-box attacks on collaborative recommendation},
  author={Zhang, Yihe and Yuan, Xu and Li, Jin and Lou, Jiadong and Chen, Li and Tzeng, Nian-Feng},
  booktitle={Proceedings of the 2021 ACM SIGSAC Conference on Computer and Communications Security},
  pages={51--68},
  year={2021}
}

@inproceedings{ma2025jailbreaking,
  title={Jailbreaking prompt attack: A controllable adversarial attack against diffusion models},
  author={Ma, Jiachen and Li, Yijiang and Xiao, Zhiqing and Cao, Anda and Zhang, Jie and Ye, Chao and Zhao, Junbo},
  booktitle={Findings of the Association for Computational Linguistics: NAACL 2025},
  pages={3141--3157},
  year={2025}
}

@inproceedings{uesato2018adversarial,
  title={Adversarial risk and the dangers of evaluating against weak attacks},
  author={Uesato, Jonathan and O’donoghue, Brendan and Kohli, Pushmeet and Oord, Aaron},
  booktitle={International Conference on Machine Learning},
  pages={5025--5034},
  year={2018},
  organization={PMLR}
}

@inproceedings{maho2021surfree,
  title={Surfree: a fast surrogate-free black-box attack},
  author={Maho, Thibault and Furon, Teddy and Le Merrer, Erwan},
  booktitle={Proceedings of the IEEE/CVF Conference on Computer Vision and Pattern Recognition},
  pages={10430--10439},
  year={2021}
}

@inproceedings{wang2022triangle,
  title={Triangle attack: A query-efficient decision-based adversarial attack},
  author={Wang, Xiaosen and Zhang, Zeliang and Tong, Kangheng and Gong, Dihong and He, Kun and Li, Zhifeng and Liu, Wei},
  booktitle={European conference on computer vision},
  pages={156--174},
  year={2022},
  organization={Springer}
}

@inproceedings{liu2017delving,
  title={Delving into Transferable Adversarial Examples and Black-box Attacks},
  author={Liu, Yanpei and Chen, Xinyun and Liu, Chang and Song, Dawn},
  booktitle={International Conference on Learning Representations},
  year={2017}
}

@inproceedings{li2023making,
  title={Making Substitute Models More Bayesian Can Enhance Transferability of Adversarial Examples},
  author={Li, Qizhang and Guo, Yiwen and Zuo, Wangmeng and Chen, Hao},
  booktitle={The Eleventh International Conference on Learning Representations},
year={2023}
}

@inproceedings{chen2024rethinking,
  title={Rethinking Model Ensemble in Transfer-based Adversarial Attacks},
  author={Chen, Huanran and Zhang, Yichi and Dong, Yinpeng and Yang, Xiao and Su, Hang and Zhu, Jun},
  booktitle={The Twelfth International Conference on Learning Representations},
year={2024}
}

@inproceedings{bai2020improving,
  title={Improving Query Efficiency of Black-Box Adversarial Attack},
  author={Bai, Yang and Zeng, Yuyuan and Jiang, Yong and Wang, Yisen and Xia, Shu-Tao and Guo, Weiwei},
  booktitle={European Conference on Computer Vision},
  pages={101--116},
  year={2020}
}

@inproceedings{rahmati2020geoda,
  title={Geoda: a geometric framework for black-box adversarial attacks},
  author={Rahmati, Ali and Moosavi-Dezfooli, Seyed-Mohsen and Frossard, Pascal and Dai, Huaiyu},
  booktitle={Proceedings of the IEEE/CVF Conference on Computer Vision and Pattern Recognition},
  pages={8446--8455},
  year={2020}
}

@inproceedings{xie2019improving,
  title={Improving transferability of adversarial examples with input diversity},
  author={Xie, Cihang and Zhang, Zhishuai and Zhou, Yuyin and Bai, Song and Wang, Jianyu and Ren, Zhou and Yuille, Alan L},
  booktitle={Proceedings of the IEEE/CVF Conference on Computer Vision and Pattern Recognition},
  pages={2730--2739},
  year={2019}
}

@inproceedings{nasee2022rimproving,
  title={On Improving Adversarial Transferability of Vision Transformers},
  author={Naseer, Muzammal and Ranasinghe, Kanchana and Khan, Salman and Khan, Fahad and Porikli, Fatih},
  booktitle={International Conference on Learning Representations},
year={2022}
}

@inproceedings{xie2025chain,
  title={Chain of Attack: On the Robustness of Vision-Language Models Against Transfer-Based Adversarial Attacks},
  author={Xie, Peng and Bie, Yequan and Mao, Jianda and Song, Yangqiu and Wang, Yang and Chen, Hao and Chen, Kani},
  booktitle={Proceedings of the Computer Vision and Pattern Recognition Conference},
  pages={14679--14689},
  year={2025}
}

@article{debicha2023tad,
  title={TAD: Transfer learning-based multi-adversarial detection of evasion attacks against network intrusion detection systems},
  author={Debicha, Islam and Bauwens, Richard and Debatty, Thibault and Dricot, Jean-Michel and Kenaza, Tayeb and Mees, Wim},
  journal={Future Generation Computer Systems},
  volume={138},
  pages={185--197},
  year={2023},
  publisher={Elsevier}
}

@article{hinton2015distilling,
  title={Distilling the Knowledge in a Neural Network},
  author={Hinton, G},
  journal={arXiv preprint arXiv:1503.02531},
  year={2015}
}

@inproceedings{wong2020fast,
  title={Fast is better than free: Revisiting adversarial training},
  author={Wong, Eric and Rice, Leslie and Kolter, J Zico},
  booktitle={International Conference on Learning Representations},
year={2020}
}

@article{shafahi2019adversarial,
  title={Adversarial training for free!},
  author={Shafahi, Ali and Najibi, Mahyar and Ghiasi, Mohammad Amin and Xu, Zheng and Dickerson, John and Studer, Christoph and Davis, Larry S and Taylor, Gavin and Goldstein, Tom},
  journal={Advances in Neural Information Processing Systems},
  volume={32},
  year={2019}
}

@inproceedings{guo2018countering,
  title={Countering Adversarial Images using Input Transformations},
  author={Guo, Chuan and Rana, Mayank and Cisse, Moustapha and van der Maaten, Laurens},
  booktitle={International Conference on Learning Representations},
  year={2018}
}

@inproceedings{nie2022diffusion,
  title={Diffusion Models for Adversarial Purification},
  author={Nie, Weili and Guo, Brandon and Huang, Yujia and Xiao, Chaowei and Vahdat, Arash and Anandkumar, Animashree},
  booktitle={International Conference on Machine Learning},
  pages={16805--16827},
  year={2022},
  organization={PMLR}
}

@article{wei2025real,
  title={Real-world adversarial defense against patch attacks based on diffusion model},
  author={Wei, Xingxing and Kang, Caixin and Dong, Yinpeng and Wang, Zhengyi and Ruan, Shouwei and Chen, Yubo and Su, Hang},
  journal={IEEE Transactions on Pattern Analysis and Machine Intelligence},
  year={2025},
  publisher={IEEE}
}

@article{nesti2021detecting,
  title={Detecting adversarial examples by input transformations, defense perturbations, and voting},
  author={Nesti, Federico and Biondi, Alessandro and Buttazzo, Giorgio},
  journal={IEEE Transactions on Neural Networks and Learning Systems},
  volume={34},
  number={3},
  pages={1329--1341},
  year={2021},
  publisher={IEEE}
}

@article{che2025large,
  title={Large Language Model Text Adversarial Defense Method Based on Disturbance Detection and Error Correction},
  author={Che, Lei and Wu, Chengcong and Hou, Yan},
  journal={Electronics},
  volume={14},
  number={11},
  pages={2267},
  year={2025},
  publisher={MDPI}
}

@inproceedings{liu2019detection,
  title={Detection based defense against adversarial examples from the steganalysis point of view},
  author={Liu, Jiayang and Zhang, Weiming and Zhang, Yiwei and Hou, Dongdong and Liu, Yujia and Zha, Hongyue and Yu, Nenghai},
  booktitle={Proceedings of the IEEE/CVF Conference on Computer Vision and Pattern Recognition},
  pages={4825--4834},
  year={2019}
}

@article{aldahdooh2022adversarial,
  title={Adversarial example detection for DNN models: A review and experimental comparison},
  author={Aldahdooh, Ahmed and Hamidouche, Wassim and Fezza, Sid Ahmed and D{\'e}forges, Olivier},
  journal={Artificial Intelligence Review},
  volume={55},
  number={6},
  pages={4403--4462},
  year={2022},
  publisher={Springer}
}

@inproceedings{yu2021defending,
  title={Defending against universal adversarial patches by clipping feature norms},
  author={Yu, Cheng and Chen, Jiansheng and Xue, Youze and Liu, Yuyang and Wan, Weitao and Bao, Jiayu and Ma, Huimin},
  booktitle={Proceedings of the IEEE/CVF International Conference on Computer Vision},
  pages={16434--16442},
  year={2021}
}

@inproceedings{allouaha2025daptive,
  title={Adaptive Gradient Clipping for Robust Federated Learning},
  author={Allouah, Youssef and Guerraoui, Rachid and Gupta, Nirupam and Jellouli, Ahmed and Rizk, Geovani and Stephan, John},
  booktitle={The Thirteenth International Conference on Learning Representations},
year={2025}
}

@article{zhao2025universal,
  title={Universal attention guided adversarial defense using feature pyramid and non-local mechanisms},
  author={Zhao, Jiawei and Xie, Lizhe and Gu, Siqi and Qin, Zihan and Zhang, Yuning and Wang, Zheng and Hu, Yining},
  journal={Scientific Reports},
  volume={15},
  number={1},
  pages={5237},
  year={2025},
  publisher={Nature Publishing Group UK London}
}

@inproceedings{lecuyer2019certified,
  title={Certified robustness to adversarial examples with differential privacy},
  author={Lecuyer, Mathias and Atlidakis, Vaggelis and Geambasu, Roxana and Hsu, Daniel and Jana, Suman},
  booktitle={2019 IEEE Symposium on Security and Privacy (SP)},
  pages={656--672},
  year={2019},
  organization={IEEE}
}

@inproceedings{yang2020randomized,
  title={Randomized smoothing of all shapes and sizes},
  author={Yang, Greg and Duan, Tony and Hu, J Edward and Salman, Hadi and Razenshteyn, Ilya and Li, Jerry},
  booktitle={International Conference on Machine Learning},
  pages={10693--10705},
  year={2020},
  organization={PMLR}
}

@inproceedings{bansal2022certified,
  title={Certified neural network watermarks with randomized smoothing},
  author={Bansal, Arpit and Chiang, Ping-yeh and Curry, Michael J and Jain, Rajiv and Wigington, Curtis and Manjunatha, Varun and Dickerson, John P and Goldstein, Tom},
  booktitle={International Conference on Machine Learning},
  pages={1450--1465},
  year={2022},
  organization={PMLR}
}

@inproceedings{korzh2025certification,
  title={Certification of speaker recognition models to additive perturbations},
  author={Korzh, Dmitrii and Karimov, Elvir and Pautov, Mikhail and Rogov, Oleg Y and Oseledets, Ivan},
  booktitle={Proceedings of the AAAI Conference on Artificial Intelligence},
  volume={39},
  number={17},
  pages={17947--17956},
  year={2025}
}

@inproceedings{mao2024understanding,
  title={Understanding Certified Training with Interval Bound Propagation},
  author={Mao, Yuhao and Mueller, Mark Niklas and Fischer, Marc and Vechev, Martin},
  booktitle={The Twelfth International Conference on Learning Representations},
year={2024}
}

@inproceedings{shi2020robustness,
  title={Robustness Verification for Transformers},
  author={Shi, Zhouxing and Zhang, Huan and Chang, Kai-Wei and Huang, Minlie and Hsieh, Cho-Jui},
  booktitle={International Conference on Learning Representations},
year={2020}
}

@inproceedings{tjeng2019evaluating,
  title={Evaluating Robustness of Neural Networks with Mixed Integer Programming},
  author={Tjeng, Vincent and Xiao, Kai Y and Tedrake, Russ},
  booktitle={International Conference on Learning Representations},
year={2019}
}

@article{wang2021beta,
  title={Beta-crown: Efficient bound propagation with per-neuron split constraints for neural network robustness verification},
  author={Wang, Shiqi and Zhang, Huan and Xu, Kaidi and Lin, Xue and Jana, Suman and Hsieh, Cho-Jui and Kolter, J Zico},
  journal={Advances in Neural Information Processing Systems},
  volume={34},
  pages={29909--29921},
  year={2021}
}

@inproceedings{dong2018boosting,
  title={Boosting adversarial attacks with momentum},
  author={Dong, Yinpeng and Liao, Fangzhou and Pang, Tianyu and Su, Hang and Zhu, Jun and Hu, Xiaolin and Li, Jianguo},
  booktitle={Proceedings of the IEEE Conference on Computer Vision and Pattern Recognition},
  pages={9185--9193},
  year={2018}
}

@article{krizhevsky2009learning,
  title={Learning multiple layers of features from tiny images},
  author={Krizhevsky, Alex and Hinton, Geoffrey and others},
  year={2009},
  publisher={Toronto, ON, Canada}
}

@inproceedings{deng2009imagenet,
  title={Imagenet: A large-scale hierarchical image database},
  author={Deng, Jia and Dong, Wei and Socher, Richard and Li, Li-Jia and Li, Kai and Fei-Fei, Li},
  booktitle={2009 IEEE Conference on Computer Vision and Pattern Recognition},
  pages={248--255},
  year={2009},
  organization={Ieee}
}

@inproceedings{
dosovitskiy2021an,
title={An Image is Worth 16x16 Words: Transformers for Image Recognition at Scale},
author={Alexey Dosovitskiy and Lucas Beyer and Alexander Kolesnikov and Dirk Weissenborn and Xiaohua Zhai and Thomas Unterthiner and Mostafa Dehghani and Matthias Minderer and Georg Heigold and Sylvain Gelly and Jakob Uszkoreit and Neil Houlsby},
booktitle={International Conference on Learning Representations},
year={2021},
}

@inproceedings{he2016deep,
  title={Deep residual learning for image recognition},
  author={He, Kaiming and Zhang, Xiangyu and Ren, Shaoqing and Sun, Jian},
  booktitle={Proceedings of the IEEE Conference on Computer Vision and Pattern Recognition},
  pages={770--778},
  year={2016}
}

@inproceedings{cheng2020sign,
  title={Sign-OPT: A Query-Efficient Hard-label Adversarial Attack},
  author={Cheng, Minhao and Singh, Simranjit and Chen, Patrick H and Chen, Pin-Yu and Liu, Sijia and Hsieh, Cho-Jui},
  booktitle={International Conference on Learning Representations},
  year={2020}
}

@article{zhou2025query,
  title={Query-efficient hard-label black-box attack against vision transformers},
  author={Zhou, Chao and Shi, Xiaowen and Wang, Yuan-Gen},
  journal={Applied Soft Computing},
  volume={183},
  pages={113686},
  year={2025},
  publisher={Elsevier}
}

@article{shi2022decision,
  title={Decision-based black-box attack against vision transformers via patch-wise adversarial removal},
  author={Shi, Yucheng and Han, Yahong and Tan, Yu-an and Kuang, Xiaohui},
  journal={Advances in Neural Information Processing Systems},
  volume={35},
  pages={12921--12933},
  year={2022}
}

@inproceedings{croce2022sparse,
  title={Sparse-rs: a versatile framework for query-efficient sparse black-box adversarial attacks},
  author={Croce, Francesco and Andriushchenko, Maksym and Singh, Naman D and Flammarion, Nicolas and Hein, Matthias},
  booktitle={Proceedings of the AAAI Conference on Artificial Intelligence},
  volume={36},
  number={6},
  pages={6437--6445},
  year={2022}
}
\bibliographystyle{icml2026}

\newpage
\appendix
\onecolumn

\section{Proof of the lemma.}

\begin{proof}
    Recall that $K$ is the number of classes and let $S_j$ be the instance of the surrogate model on $j$-th alternation iteration.  Let $\{z_j\}_{j=1}^\infty$ be the sequence of adversarial examples, where $z_j$ is an adversarial example for $S_j$ and $z_0=x.$ Note that $z_j \in U_\delta(x)$ for all $j \in \mathbb{Z}_{+}.$ 
    Since for all $j \in \mathbb{Z}_+$, $S_j$ is differentiable within $U_\delta(x),$  for any two points $a,b \in U_\delta(x)$ we may write 
    \begin{equation}
        S_j(a) - S_j(b) = \nabla S_j(\tau)^\top (a-b),
    \end{equation}
    where $\tau \in U_\delta(x)$ and on the line segment between $a$ and $b$.  Specifically, for two subsequent adversarial examples, $z_j$ and $z_{j-1},$ the expression becomes 
    \begin{equation}
         S_j(z_j) - S_j(z_{j-1}) = \nabla S_j({\tau}_j)^\top (z_j-z_{j-1}),  
    \end{equation}
where $\tau_j$  is on the line segment between $z_j$ and $z_{j-1}.$ Note that $z_j$ is adversarial example for $S_j$, whereas $z_{j-1}$ was included into distillation dataset $\mathcal{D}(S)$ on previous alternation iteration.

By introducing $\rho_j = \|z_j - z_{j-1}\|_\infty,$ one can see that 
\begin{align}
    & \rho_j \le t\rho_{j-1} \le t^2 \rho_{j-2} \le \dots \le t^{j-1}\rho_1 = \nonumber \\ &t^{j-1}\|z_1 - z_0\|_\infty = t^{j-1}\|z_1 - x\|_\infty \le t^{j-1}\delta.
\end{align}
Note that when $\rho_j$ is less than $\varepsilon / \gamma,$ the norm of the difference between $S_j(z_j) - S_j(z_{j-1})$ is  bounded from above. Specifically, let 
\begin{equation}
    \phi:[0,1] \to [0,1]^K, \quad \phi(t) = S_j(z_j + t(z_{j-1}-z_j))
\end{equation}
and 
\begin{equation}
    \phi'(t) = \nabla S_j(z_j + t(z_{j-1} - z_j))(z_j - z_{j-1}).
\end{equation}
Thus,
\begin{equation}
    S_j(z_{j-1}) - S_j(z_{j}) = \phi(1) - \phi(0) = \int_{0}^{1} \phi'(t)dt = \int_0^1 \nabla S_j(z_j + t(z_{j-1} - z_j))(z_j - z_{j-1}) dt. 
\end{equation}

Now, since 

\begin{equation}
    \|\nabla S_j \|_{op,\infty} = \sup_{\|x\| \ne 0} \frac{\|\nabla S_j x\|_\infty}{\|x\|_\infty} \implies \|\nabla S_j x\|_\infty \le \|\nabla S_j \|_{op,\infty}  \|x\|_\infty 
\end{equation}
we write 

\begin{align}
    \|S_j(z_{j-1}) - S_j(z_{j})\|_\infty     =\left\|\int_0^1 \nabla S_j(z_j + t(z_{j-1} - z_j))(z_j - z_{j-1}) dt \right\|_\infty \le \nonumber \\
     \le \int_0^1\left\| \nabla S_j(z_j + t(z_{j-1} - z_j))(z_j - z_{j-1})  \right\|_\infty dt \le \nonumber \\
    \le \int_0^1 \|\nabla S_j(z_j + t(z_{j-1} - z_j)) \|_{op, \infty} \|z_{j} - z_{j-1}\|_\infty dt  \le \nonumber \\
    \le \|z_j - z_{j-1}\|_\infty \int_0^1 \sup \big(\|\nabla S_j(z_j + t(z_{j-1} - z_j))\|_{op,\infty}\big)  dt \le \nonumber \\
    \le \gamma \|z_j - z_{j-1}\|_\infty < \gamma \varepsilon/\gamma = \varepsilon
\end{align}

That yields 
\begin{equation}
\label{eq:same_indices}
\arg\max_{i \in [1,\dots,K]}S_j(z_j)_i = \arg\max_{i \in [1,\dots,K]}S_j(z_{j-1})_i
\end{equation}
according to Eq. \ref{eq:half_difference}. Recall that  $z_{j-1}$ is included into distillation dataset  $\mathcal{D}(S)$ on iteration $j-1,$ so  
\begin{equation}
 \arg\max_{i \in [1,\dots,K]}S_j(z_{j-1})_i = T(z_{j-1})
\end{equation} according to Eq. \ref{eq:student_gap}.
Now observe that Algorithm \ref{alg:adv_whitebox} yielded an adversarial example, namely, $z_{j}$ for the model $S$ on iteration $j$. At the same time, the prediction of $S$ for $z_{j-1}$ and for $z_j$ are the same (see Eq. \ref{eq:same_indices}). That means that the predicted class label for $z_j$, say, $c_A = \arg\max_{i \in [1,\dots,K]}S_j(z_j)_i $ was assigned by $T$ to the previous sample, $z_{j-1}$:
\begin{equation}
    c_A = \arg\max_{i \in [1,\dots,K]}S_j(z_j)_i = T(z_{j-1}).
\end{equation}
Finally, for the values of $j$ satisfying
\begin{equation}
    t^{j-1}\delta \le \varepsilon / \gamma \ \xleftrightarrow{t \in (0,1)} \ (j-1)\ln t \le \ln \varepsilon - \ln \delta - \ln \gamma, 
\end{equation}
the value $\rho_j$ is less than $\varepsilon /\gamma,$ what finalizes the proof. 
    
\end{proof}

\section{Hyperparameters of Baseline Methods}
In this section, we present the values of hyperparameters used in the methods with which we compare our approach. In all experiments, the query budget was set to 500. The only exception is the ResNet50 model on the CIFAR-10 dataset, where the query budget was limited to 300.

\begin{table}[htb!]
  \caption{Hyperparameters of baseline methods}
  \label{tab:hyperparams}
  \begin{center}
      \begin{sc}
        \begin{tabular}{@{}ll@{}}
        \toprule
        \textbf{Method} & \textbf{Hyperparameters} \\ \midrule
        HopSkipJump $l_2$ / $l_\infty$& 
        \begin{tabular}[c]{@{}l@{}}
            num\_samples\_for\_init = $100$ \\
            num\_samples\_for\_grad\_est = $100$ \\
            max\_iter=$100$
            
        \end{tabular} \\ \midrule
        
        SignOPT &
        \begin{tabular}[c]{@{}l@{}}
            num\_samples\_for\_init = $100$ \\
            num\_samples\_for\_grad\_est = $100$ \\
            max\_iter = $100$ \\
        \end{tabular} \\ \midrule
        
        GeoDA &
        \begin{tabular}[c]{@{}l@{}}
            sub\_dim = 150 \\
            db\_search\_steps = 200 \\
            bin\_search\_tol = 0.0001 \\
            $\lambda$ = 0.6 \\
            $\sigma$ = 0.0002 \\
        \end{tabular} \\ \midrule
        
        PAR &
        \begin{tabular}[c]{@{}l@{}}
            initial\_patch\_size = $56$ \\
            min\_patch\_size = $7$ \\
        \end{tabular} \\ \midrule

        AdvViT &
        \begin{tabular}[c]{@{}l@{}}
            num\_samples\_for\_init = $100$ \\
            init\_attempts\_extra = $100$ \\
            patch\_num = $14$ \\
            dim\_size = $4$ \\
            $\alpha$ = $4.0$ \\
            K\_sign = $100$ \\
        \end{tabular} \\ \midrule  
        
        SquareAttack $l_\infty$ &
        \begin{tabular}[c]{@{}l@{}}
            $\varepsilon = \frac{32}{255}$ \\
            p\_init = $0.05$ \\
        \end{tabular} \\ \midrule
        
        SparseRS &
        \begin{tabular}[c]{@{}l@{}}
            norm = $\ell_0$ \\
            $\varepsilon$ = $2000$ \\
            p\_init = $0.3$ \\
        \end{tabular} \\ 
        

        
        \bottomrule
        \end{tabular}
      \end{sc}
  \end{center}
\end{table}



\end{document}